%% file: main.tex
\documentclass[twoside]{article}

\usepackage{styletwocol}
\usepackage[utf8]{inputenc}
\usepackage{graphicx,amsmath,amsfonts,amssymb,bm,breakurl,epsfig,epsf,color,MnSymbol,mathbbol,fmtcount,algorithm,semtrans,caption,subcaption,
  multirow,bm}

\usepackage[title]{appendix}
\usepackage{caption}
\usepackage[bottom,hang,flushmargin]{footmisc}
\usepackage{comment}
\usepackage[bottom]{footmisc}
\usepackage{enumitem}
\usepackage{algpseudocode}
\usepackage{authblk}
\usepackage{amsthm}
\usepackage{thmtools}
\usepackage{thm-restate}
\usepackage{float}
\restylefloat{table}
\usepackage{flafter} 
\usepackage{wrapfig}
\usepackage{subcaption} 
\input{math_commands.tex}
\usepackage[round]{natbib}

\newtheoremstyle{break}
  {}
  {}
  {\itshape}
  {}
  {\bfseries}
  {.}
  {\newline}
  {}

\newtheorem{theorem}{Theorem}
\newtheorem{lemma}{Lemma}
\newtheorem{corollary}{Corollary}
\newtheorem{theoremrestate}{Theorem}
\newtheorem{lemmarestate}{Lemma}

\newtheorem*{proposition*}{Proposition}

\newtheorem{definition}{Definition}[section]

\begin{document}

%

%

\twocolumn[

\aistatstitle{Wasserstein Smoothing: Certified Robustness against Wasserstein Adversarial Attacks}

\aistatsauthor{ Alexander Levine and Soheil Feizi }

\aistatsaddress{University of Maryland, College Park\\ \{alevine0, sfeizi\}@cs.umd.edu } ]

\begin{abstract}
In the last couple of years, several adversarial attack methods based on different threat models have been proposed for the image classification problem. Most existing defenses consider additive threat models in which sample perturbations have bounded $L_p$ norms. These defenses, however, can be vulnerable against adversarial attacks under non-additive threat models. An example of an attack method based on a non-additive threat model is the Wasserstein adversarial attack proposed by \cite{pmlr-v97-wong19a}, where the distance between an image and its adversarial example is determined by the Wasserstein metric (``earth-mover distance'') between their normalized pixel intensities. Until now, there has been no certifiable defense against this type of attack. In this work, we propose the first defense with certified robustness against Wasserstein Adversarial attacks using randomized smoothing. We develop this certificate by considering the space of possible flows between images, and representing this space such that Wasserstein distance between images is upper-bounded by $L_1$ distance in this flow-space. We can then apply existing randomized smoothing certificates for the $L_1$ metric.  In MNIST and CIFAR-10 datasets, we find that our proposed defense is also practically effective, demonstrating significantly improved accuracy under Wasserstein adversarial attack compared to unprotected models.

\end{abstract}
\section{Introduction}
In recent years, adversarial attacks against machine learning systems, and defenses against these attacks, have been heavily studied \citep{szegedy2013intriguing, madry2017towards, carlini2017towards}. Although these attacks have been applied in a variety of domains, image classification tasks remain a major focus of research. In general, for a specified image classifier $\vf$, the goal of an adversarial attack on an image $\vx$ is to produce a perturbed image $\tbx$ that is imperceptibly `close' to $\vx$, such that $\vf$ classifies $\tbx$ differently than $\vx$. This `closeness' notion can be measured in a variety of different ways under different threat models. Most existing attacks and defenses consider additive threat models where the $L_p$ norm of $\tbx-\vx$ is bounded. 

Recently, non-additive threat models \citep{pmlr-v97-wong19a, laidlaw2019functional, engstrom2018rotation, assion2019attack} have been introduced which aim to minimize the distance between $\vx$ and $\tbx$ according to other metrics. Among these attacks is the attack introduced by \cite{pmlr-v97-wong19a} which considers the Wasserstein distance between $\vx$ and $\tbx$, normalized such that the pixel intensities of the image can be treated as probability distributions. Informally, the Wasserstein distance between probability distributions $\vx$ and $\tbx$ measures the minimum cost to `transport' probability mass in order to transform $\vx$ into $\tbx$, where the cost scales with both the amount of mass transported and the distance over which it is transported with respect to some underlying metric. The intuition behind this threat model is that shifting pixel intensity a short distance across an image is less perceptible than moving the same amount of pixel intensity a larger distance (See Figure \ref{fig:wattack} for an example of a Wasserstein adversarial attack.) 

\begin{figure}[t]
    \centering
    \includegraphics[width=0.45\textwidth]{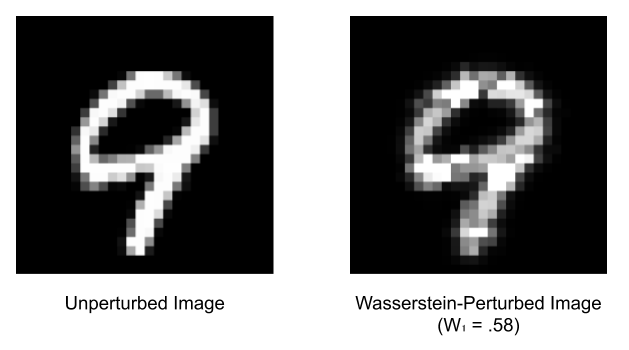}
    \caption{An illustration of Wasserstein adversarial attack \citep{pmlr-v97-wong19a}. }
    \label{fig:wattack}
\end{figure}
A variety of practical approaches have been proposed to make classifiers robust against adversarial attack including adversarial training \citep{madry2017towards}, defensive distillation \citep{papernot2016distillation}, and obfuscated gradients \citep{papernot2017practical}. However, as new defenses are proposed, new attack methodologies are often rapidly developed which defeat these defences \citep{tramer2017ensemble,athalye2018obfuscated,carlini2016defensive}. While updated defences are often then proposed \citep{tramer2017ensemble}, in general, we cannot be confident that new attacks will not in turn defeat these defences.

To escape this cycle, approaches have been proposed to develop certifiably robust classifiers \citep{wong2018provable, gowal2018effectiveness, lecuyer2018certified, li2018second, cohen2019certified, salman2019provably}: in these classifiers, for each image $\vx$, one can calculate a radius $\rho$ such that it is provably guaranteed that any other image $\tbx$ with distance less than $\rho$ from $\vx$ will be classified similarly to $\vx$. This means that no adversarial attack can ever be developed which produces adversarial examples to the classifier within the certified radius.

One effective approach to develop certifiably robust classification is to use randomized smoothing with a probabilistic robustness certificate \citep{lecuyer2018certified, li2018second, cohen2019certified, salman2019provably}. In this approach, one uses a smoothed classifier $\barbf(\vx)$, which represents the expectation of $\vf(\vx)$ over random perturbations of $\vx$. Based on this smoothing, one can derive an upper bound on how steeply the scores assigned to each class by $\barbf$ can change, which can then be used to derive a radius $\rho$ in which the highest class score must remain highest\footnote{In practice, samples are used to estimate the expectation $\barbf(\vx)$, producing an empirical smoothed classifier $\tbf(\vx)$: the certification is therefore probabilistic, with a degree of certainty dependent on the number of samples.}.

In this work, we present the first certified defence against Wasserstein adversarial attacks using an adapted randomized smoothing approach, which we call \textit{Wasserstein smoothing}. To develop the robustness certificate, we define a (non-unique) representation of the difference between two images, based on the flow of pixel intensity necessary to construct one image from another.  In this representation, we show that the $L_1$ norm of the minimal flow between two images is equal to the Wasserstein distance between the images. This allows us to apply existing $L_1$ smoothing-based defences, by adding noise in the space of these representations of flows. We show that empirically that this gives improved robustness certificates, compared to using a weak upper bound of Wasserstein distance given by randomized smoothing in the feature space of images directly. We also show that our Wasserstein smoothing defence protects against Wasserstein adversarial attacks empirically, with significantly improved robustness compared to baseline models. For small adversarial perturbations on the MNIST dataset, our method achieves higher accuracy under adversarial attack than all existing practical defences for the Wasserstein threat model.

In summary, we make the following contributions:
\begin{itemize}
    \item We develop a novel certified defence for the Wasserstein adversarial attack threat model. This is the first certified defence, to our knowledge, that has been proposed for this threat model.
    \item We demonstrate that our certificate is nonvacuous, in that it can certify Wasserstein radii larger than those which can be certified by exploiting a trivial $L_1$ upper bound on Wasserstein distance.
    \item We demonstrate that our defence effectively protects against existing Wasserstein adversarial attacks, compared to an unprotected baseline. 
\end{itemize}


\section{Background}
Let $\vx\in[0,1]^{n \times m}$ denote a two dimensional image, of height $n$ and width $m$. We will normalize the image such that $\sum_i \sum_j x_{i,j} = 1$, so that $\vx$ can be interpreted as a probability distribution on the discrete support of pixel coordinates of the two-dimensional image.\footnote{In the case of multi-channel color images, the attack proposed by \cite{pmlr-v97-wong19a} does not transport pixel intensity between channels. This allows us to defend against these attacks using our 2D Wasserstein smoothing with little modification. See Section \ref{sec:color}, and Corollary \ref{threeD} in the appendix} Following the notation of \cite{pmlr-v97-wong19a}, we define the p-Wasserstein distance between $\vx$ and $\vx'$ as:
\begin{definition}
Given two distributions $\vx, \vx' \in[0,1]^{n \times m} $, and a distance metric $d \in ([n] \times [m]) \times ([n] \times [m]) \rightarrow \R$ , the p-Wasserstein distance as: 
\begin{align} \label{wassdef}
    W_p(\vx,\vx') = & \min_{\Pi \in \R_{+}^{(n\cdot m) \times (n\cdot m)}} <\Pi,C>, \\
    &\Pi\1=\vx, \,\, \Pi^T\1=\vx', \nonumber\\
    & C_{(i,j), (i',j')} := \left[d\left((i,j),(i',j')\right)\right]^p. \nonumber
\end{align}
$C_{(i,j), (i',j')}$ is the cost of transporting a mass unit from the position $(i,j)$ to $(i',j')$ in the image.
\end{definition}
Note that, for the purpose of matrix multiplication, we are treating $\vx,\vx'$ as vectors of length $nm$. Similarly, the transport plan matrix $\Pi$ and the cost matrix $C$ are in $\RR^{nm\times nm}$.

Intuitively, $\Pi_{(i,j), (i',j')}$ represents the amount of probability mass to be transported from pixel $(i,j)$ to $(i',j')$, while  $C_{(i,j), (i',j')}$ represents the cost per unit probability mass to transport this probability.
We can choose $d(.,.)$ to be any measure of distance between pixel positions in an image. For example, in order to represent the $L_1$ distance metric between pixel positions, we can choose:
\begin{equation}
   d\left((i,j),(i',j')\right) = |i-i'|+|j-j'|.
\end{equation}
Moreover, to represent the $L_2$ distance metric between pixel positions, we can choose:
\begin{equation}
  d\left((i,j),(i',j')\right) = \sqrt{(i-i')^2+(j-j')^2}.
\end{equation}
Our defence directly applies to the 1-Wasserstein metric using the $L_1$ distance as the metric $d(.,.)$, while the attack developed by  \cite{pmlr-v97-wong19a} uses the $L_2$ distance. However, because images are two dimensional, these differ by at most a constant factor of $\sqrt{2}$, so we adapt our certificates to the setting of \cite{pmlr-v97-wong19a} by simply scaling our certificates by $1/\sqrt{2}$. All experimental results will be presented with this scaling. We emphasize that this it {\it not} the distinction between 1-Wasserstein and 2-Wasserstein distances: this paper uses the 1-Wasserstein metric, to match the majority of the experimental results of \cite{pmlr-v97-wong19a}.

To develop our certificate, we rely an alternative linear program formulation for the 1-Wasserstein distance on a two-dimensional image with the $L_1$ distance metric, provided by \cite{ling2007efficient}:
\begin{align} \label{wassl1lp}
    W_1(\vx,\vx') =\min_{\bg} \sum _{\substack{(i,j)}}~ \sum _{\substack{(i',j')\in \cN\left(i,j\right)}} \bg_{(i,j),(i',j')}
\end{align}    
where $\bg \geq 0$ and $\forall (i,j)$, 
\begin{align}  
    \sum _{\substack{(i',j')\in \cN\left(i,j\right)}}  \bg_{(i,j),(i',j')} - \bg_{(i',j'),(i,j)} = \bx'_{i,j} -\bx_{i,j}\nonumber
\end{align}
Here, $\cN\left(i,j\right)$ denotes the (up to) four immediate (non-diagonal) neighbors of the position $(i,j)$; in other words, $\cN\left(i,j\right) = \{(i',j')\,\,|\,\,|i-i'| + |j-j'| = 1\}$. For the $L_1$ distance in two dimensions, \cite{ling2007efficient} prove that this formulation is in fact equivalent to the linear program given in Equation \ref{wassdef}. Note that only elements of  $\bg$ with $|i-i'| + |j-j'| = 1$ need to be defined: this means that the number of variables in the linear program is approximately $4nm$, compared to the $n^2m^2$ elements of $\Pi$ in Equation \ref{wassdef}. While this was originally used to make the linear program more tractable to be solved directly, we exploit the form of this linear program to devise a randomized smoothing scheme in the next section.
\section{Robustness Certificate}
\begin{figure}[t]
    \centering
    \includegraphics[width=0.40\textwidth]{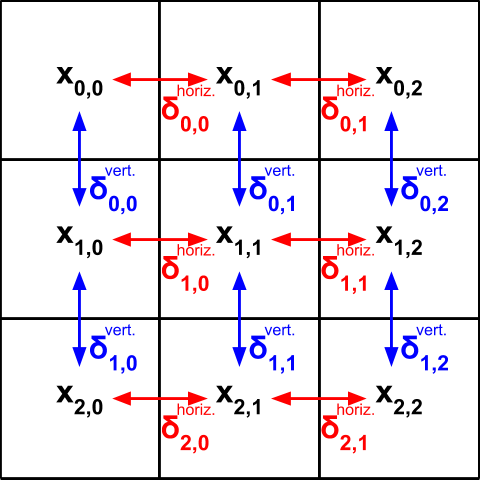}
    \caption{Indexing of the elements of the local flow map $\bm{\delta}$, in relation to the pixels of the image $\vx$, with n=m=3. }
    \label{fig:loc-flow}
\end{figure}
In order to present our robustness certificate, we first introduce some notation. Let $\bm{\delta} = \{\bm{\delta}^{\text{vert.}}\in \R^{(n-1) \times m}, \bm{\delta}^{\text{horiz.}} \in \R^{n \times (m-1)} \}$ denote a \textit{local flow plan}. It specifies a net flow between adjacent pixels in an image $\vx$, which, when applied, transforms $\vx$ to a new image $\vx'$. See Figure 2 for an explanation of the indexing. For compactness, we write $\bm{\delta} \in  \R^{r}$ where $r=(n-1)m +n(m-1)\approx 2nm$, and in general refer to the space of possible local flow plans as the \textit{flow domain}. We define the function $\Delta$, which applies a local flow to a distribution.
\begin{definition} \label{flowdef} The local flow plan application function $ \Delta \in \R^{n \times m} \times \R^{r} \rightarrow \R^{n \times m}$ is defined as:
\begin{equation} 
    \Delta(\vx,\bm{\delta})_{i,j} = \vx_{i,j} + \bm{\delta}^{\text{vert.}}_{i-1, j} -  \bm{\delta}^{\text{vert.}}_{i, j}  + \bm{\delta}^{\text{horiz.}}_{i, j-1} -  \bm{\delta}^{\text{horiz.}}_{i, j}
\end{equation}
where we let $\bm{\delta}^{\text{vert.}}_{0, j} = \bm{\delta}^{\text{vert.}}_{n, j} = \bm{\delta}^{\text{horiz.}}_{i,0} = \bm{\delta}^{\text{horiz.}}_{i, m} = 0$.\footnote{Note that the new image $\vx' = \Delta(\vx,\bm{\delta})$ is not necessarily a  probability distribution because it may have negative components. However, note that normalization is preserved: $\sum_i \sum_j x'_{i,j} = 1$. This is because every component of $\bm{\delta}$ is added once and subtracted once to elements in $\vx$.}
\end{definition}

Note that local flow plans are additive:
\begin{equation} \label{flowadd}
    \Delta(\Delta(\vx,\bm{\delta}),\bm{\delta}') = \Delta(\vx,\bm{\delta} + \bm{\delta}')
\end{equation}
Using this notation, we make a simple transformation of the linear program given in Equation \ref{wassl1lp}, removing the positivity constraint from the variables and reducing the number of variables to $\sim 2nm$:
\begin{lemma} \label{wassflow}
For any normalized probability distributions $\vx, \vx' \in[0,1]^{n \times m}$:
\begin{equation}
   W_{1}(\vx,\vx')= \min_{ \bm{\delta}: \,\, \vx' =  \Delta(\vx,\bm{\delta}) } \|\bm{\delta}\|_1 
\end{equation}
where $W_1$ denotes the 1-Wasserstein metric, using the $L_1$ distance as the underlying distance metric $d$. 
\end{lemma}
In other words, we can upper-bound the Wasserstein distance between two images using the $L_1$ norm of any feasible local flow plan between the two images. This enables us to extend existing results for $L_1$ smoothing-based certificates \citep{lecuyer2018certified} to the Wasserstein metric, by adding noise in the flow domain. 
\begin{definition}
We denote by $\mathcal{L}(\sigma) =  \laplace(0, \sigma)^{r}$ as the Laplace noise with parameter $\sigma$ in the flow domain of dimension $r$.
\end{definition}
Given a classification score function $\vf:  \R^{n \times m} \rightarrow [0,1]^k $, we define $\barbf$ as the \textit{Wasserstein-smoothed} classification function as follows:
\begin{equation}
   \barbf = \underset{\bm{\delta} \sim \mathcal{L}(\sigma)}{\E}\left[ \vf (\Delta(\vx, \bm{\delta}))\right].
\end{equation}

Let $i$ be the class assignment of $\vx$ using the Wasserstein-smoothed classifier $\barbf$ (i.e. $i = \arg\max_{i'} \barbf_{i'}(\vx)$).
\begin{theorem} \label{smoothwas}
For any normalized probability distribution $\vx \in [0,1]^{n \times m}$, if
\begin{equation}
    \barbf_{i}(\vx) \geq e^{2\sqrt{2}\rho / \sigma} \max_{i' \neq i} \barbf_{i'}(\vx) \label{thm1cond}
\end{equation}
then for any perturbed probability distribution $\tbx$ such that $W_1(\vx,\tbx) \leq \rho$, we have:
\begin{equation}
    \barbf_{i}(\tbx) \geq \max_{i' \neq i} \barbf_{i'}(\tbx).
\end{equation}
\end{theorem}
All proofs are presented in the appendix.
\section{Intuition: One-Dimensional Case}
\begin{figure}[t]
    \centering
    \includegraphics[width=0.5\textwidth]{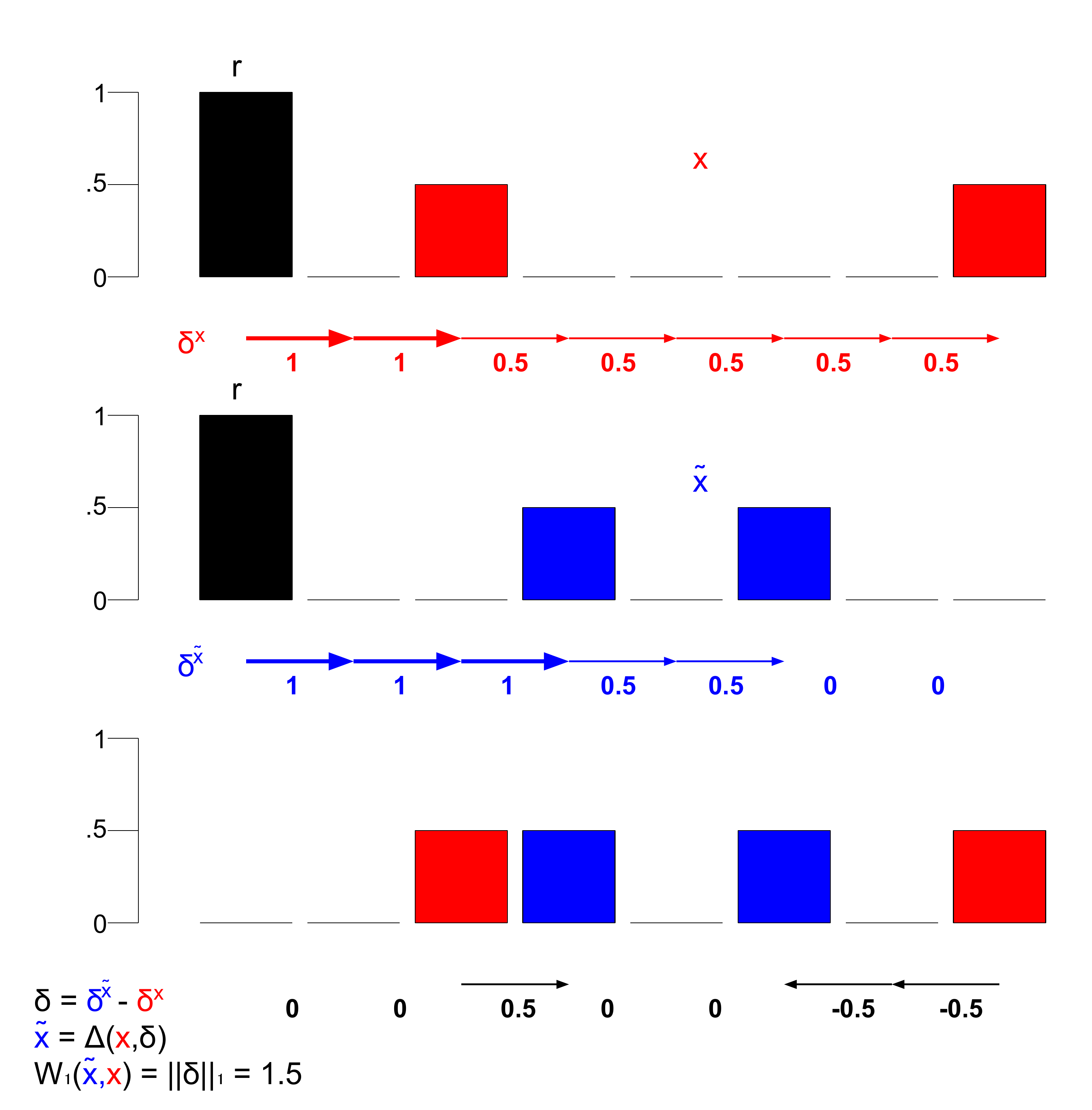}
    \caption{An illustrative example in one dimension. $\vr$ (black) denotes a fixed reference distribution. With this starting distribution fixed, $\vx$ (red) and $\tbx$ (blue) can both be uniquely represented in the flow domain as $\bm{\delta}^{\vx}$ and  $\bm{\delta}^{\tbx}$. Note that the Wasserstein distance between $\vx$ and $\tbx$ is then equivalent to the $L_1$ distance between  $\bm{\delta}^{\vx}$ and  $\bm{\delta}^{\tbx}$. In the one-dimensional case, this shows that we can transform the samples into a space where the Wasserstein threat model is equivalent to the $L_1$ metric. We can then use a pre-existing $L_1$ certified defence in the flow space to defend our classifier. }
    \label{fig:1dex}
\end{figure}
To provide an intuition about the proposed Wasserstein smoothing certified robustness scheme, we consider a simplified model, in which the support of $\vx$ is a one-dimensional array of length $n$, rather than a two-dimensional grid (i.e. $\vx \in \R^n$). In this case, we can denote a local flow plan $\bm{\delta} \in \R^{n-1}$,  so that for $\vx' = \Delta(\vx,\bm{\delta})$:
\begin{equation} \label{flowdef1d}
    \vx'_{i} = \vx_{i} + \bm{\delta}_{i-1} -  \bm{\delta}_{i} 
\end{equation}
where $\bm{\delta}_{0} = \bm{\delta}_{n} = 0$. In this one-dimensional case, for any fixed  $\vx, \vx'$  (with the normalization constraint that  $\sum_i x_{i} = \sum_i x'_{i} = 1$), there is a unique solution $\bm{\delta}$ to $\vx' = \Delta(\vx,\bm{\delta})$:
\begin{equation} \label{flowsol1d}
    \bm{\delta}_{i} =  \sum_{j=1}^i x_{j} - \sum_{j=1}^i x'_{j} 
\end{equation}
Note at this reminds us a well-known identity describing optimal transport between two distributions $X,Y$ which share a continuous, one-dimensional support (see section 2.6 of \cite{peyre2019computational}, for example):
\begin{equation}
    W_1(X,Y) = \int\limits_{-\infty}^{\infty} |F_X(z) - F_Y(z) |dz
\end{equation}
where $F_X, F_Y$ denote cumulative density functions. If we apply this result to our discretized case, with the index $i$ taking the place of $z$, and apply the identity to $\vx$ and $\vx'$, this becomes:
\begin{equation} \label{1dwasseq}
    W_1(\vx,\vx') = \sum_{i=1}^{n} \left|\sum_{j=1}^i x_{j} - \sum_{j=1}^i x'_{j}  \right| =  \sum_{i=1}^{n} \left|\bm{\delta}_i \right| = \|\bm{\delta}\|_1
\end{equation}
By the uniqueness of the solution given in Equation \ref{flowsol1d}, for any $\vx$, we can define $\bm{\delta}^{\vx}$ as the solution to  $\vx = \Delta(\vr,\bm{\delta})$, where $\vr$ is an arbitrary fixed reference distribution (e.g. suppose $r_1 = 1, r_{i} = 0$ for $i\neq 1$).
Therefore, instead of operating on the images $\vx, \tbx \in \R^n$ directly, we can equivalently operate on $\bm{\delta}^{\vx}$ and $\bm{\delta}^{\tbx}$ in the flow domain instead. We will therefore define a flow-domain version of our classifier $\vf$:

\begin{equation}
    \vf^{\text{flow}}(\bm{\delta}) := \vf(\Delta(\vr, \bm{\delta})).
\end{equation}
We will now perform classification entirely in the flow-domain, by first calculating $\bm{\delta}^{\vx}$ and then using $\vf^{\text{flow}}(\bm{\delta}^{\vx})$ as our classifier. Now, consider $\vx$ and an adversarial perturbation $\tbx$, and let $\bm{\delta}$ be the unique solution to $\tbx = \Delta(\vx,\bm{\delta})$. By equation \ref{1dwasseq}, $\|\bm{\delta}\|_1 = W_1(\vx,\tbx)$. Then:
\begin{equation}
    \tbx = \Delta(\vx,\bm{\delta}) =  \Delta(\Delta(\vr, \bm{\delta}^{\vx}),\bm{\delta}) = \Delta(\vr, \bm{\delta}^{\vx} + \bm{\delta})
\end{equation}
where the second equality is by equation \ref{flowadd}. Moreover, by the uniqueness of Equation \ref{flowsol1d}, $\bm{\delta}^{\tbx} = \bm{\delta}^{\vx} + \bm{\delta}$, or $\bm{\delta}^{\tbx} - \bm{\delta}^{\vx} = \bm{\delta}$. Therefore 
\begin{equation}
    \|\bm{\delta}^{\tbx} - \bm{\delta}^{\vx} \|_1 = W_1(\vx,\tbx).
\end{equation}
In other words, if we classify in the flow-domain, using $\vf^{\text{flow}}$, the $L_1$ distance between point $\bm{\delta}^{\vx}, \bm{\delta}^{\tbx}$ is the Wasserstein distance between the distributions $\vx$ and $\tbx$. Then, we can perform smoothing in the flow-domain, and use the existing $L_1$ robustness certificate provided by \cite{lecuyer2018certified}, to certify robustness.
Extending this argument to two-dimensional images adds some complication: images can no longer be represented uniquely in the flow domain, and the relationship between $L_1$ distance and the Wasserstein distance is now an upper bound. Nevertheless, the same conclusion still holds for 2D images as we state in Theorem \ref{smoothwas}. Proofs for the two-dimensional case are given in the appendix.
\begin{table*} [h!]
\centering
\begin{tabular}{|c|c|c|c|}
\hline
Noise & \textbf{Wasserstein Smoothing}&\textbf{Wasserstein Smoothing}&\textbf{Wasserstein Smoothing}\\

standard deviation& Classification accuracy&Median certified& Base Classifier\\
$\sigma$&(Percent abstained)&robustness& Accuracy\\
\hline
0.005&98.71(00.04)&0.0101&97.94\\
0.01&97.98(00.19)&0.0132&94.95\\
0.02&93.99(00.58)&0.0095&79.72\\
0.05&74.22(03.95)&0&43.67\\
0.1&49.41(01.29)&0&30.26\\
0.2&31.80(08.40)&N/A&25.13\\
0.5&22.58(00.84)&N/A&22.67\\
\hline
Noise &\textbf{Laplace Smoothing}&\textbf{Laplace Smoothing}&\textbf{Laplace Smoothing}\\
standard deviation& Classification accuracy&Median certified& Base Classifier\\
$\sigma$&(Percent abstained)&robustness& Accuracy\\
\hline
0.005&98.87(00.06)&0.0062&97.47\\
0.01&97.44(00.19)&0.0053&89.32\\
0.02&91.11(01.29)&0.0030&67.08\\
0.05&61.44(07.45)&0&33.80\\
0.1&34.92(09.36)&N/A&25.56\\
0.2&24.02(05.67)&N/A&22.85\\
0.5&22.57(01.05)&N/A&22.70\\
\hline
\end{tabular}
\caption{Certified Wasserstein Accuracy of Wasserstein and Laplace smoothing on MNIST} \label{MNISTdata}
\end{table*}
\section{Practical Certification Scheme}
To generate probabilistic robustness certificates from randomly sampled evaluations of the base classifier $\vf$, we adapt the procedure outlined by \cite{cohen2019certified} for $L_2$ certificates. We consider a \textit{hard smoothed classifier} approach: we set $\vf_j(\bx) = 1$ if the base classifier selects class $j$ at point $\bx$, and $\vf_j(\bx) = 0$ otherwise. We also use a stricter form of the condition given as Equation \ref{thm1cond}:
\begin{equation}
    \barbf_{i}(\vx) \geq e^{2\sqrt{2}\rho / \sigma} (1- \barbf_{i}(\vx) )
\end{equation}
This means that we only need to provide a probabilistic lower bound of the expectation of the largest class score, rather than bounding every class score. This reduces the number of samples necessary to estimate a high-confidence lower bound on $\barbf_{i}(\vx)$, and therefore to estimate the certificate with high confidence. \cite{cohen2019certified} provides a statistically sound procedure for this, which we use: refer to that paper for details. Note that, when simply evaluating the classification given by $\barbf(\bx)$, we will also need to approximate $\barbf$ using random samples. \cite{cohen2019certified} also provides a method to do this which yields the expected classification with high confidence, but may abstain from classifying. We will also use this method when evaluating accuracies.

Since the Wasserstein adversarial attack introduced by \cite{pmlr-v97-wong19a} uses the $L_2$ distance metric, to have a fair performance evaluation against this attack, we are interested in certifying a radius in the 1-Wasserstein distance with underlying $L_2$ distance metric, rather than $L_1$. Let us denote this radius as $\rho_{2}$. In two-dimensional images, the elements of the cost matrix $C$ in this metric may be smaller by up to a factor of $\sqrt{2}$, so we have:
\begin{equation}
    \rho_2 \geq \frac{1}{\sqrt{2} }\rho 
\end{equation}
Therefore, by certifying to a radius of $\rho = \sqrt{2}\rho_2$, we can effectively certify against the $L_2$ metric 1-Wasserstein attacks of radius $\rho_2$; our condition then becomes:
\begin{equation}
    \barbf_{i}(\vx) \geq e^{4\rho_2 / \sigma} (1- \barbf_{i}(\vx) ).
\end{equation}
\begin{figure*}[h!]
\begin{minipage}[b]{0.5\textwidth}
    \includegraphics[width=\textwidth]{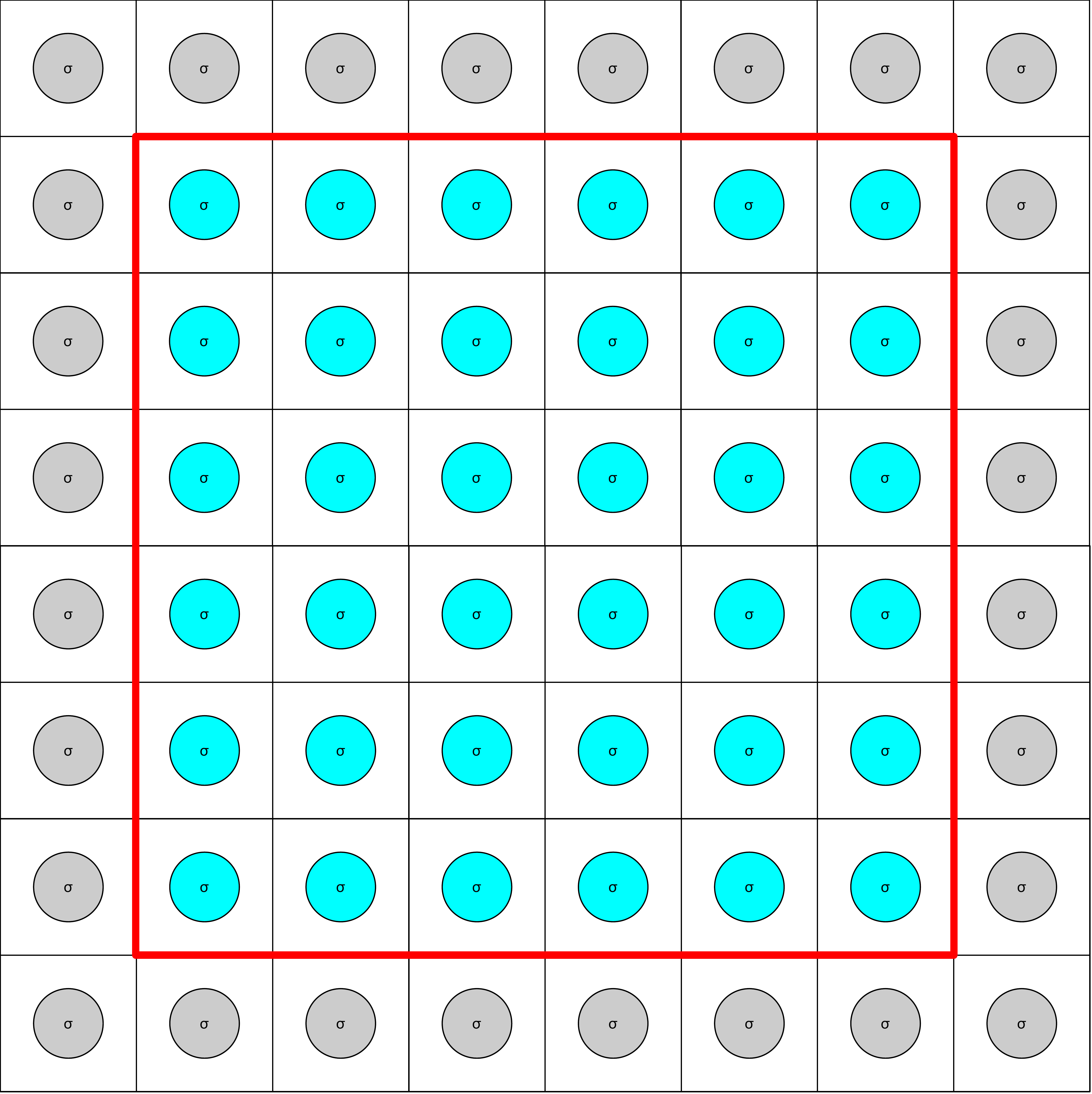}
    (a) Laplace Smoothing
\end{minipage}
\begin{minipage}[b]{0.5\textwidth}
    \includegraphics[width=\textwidth]{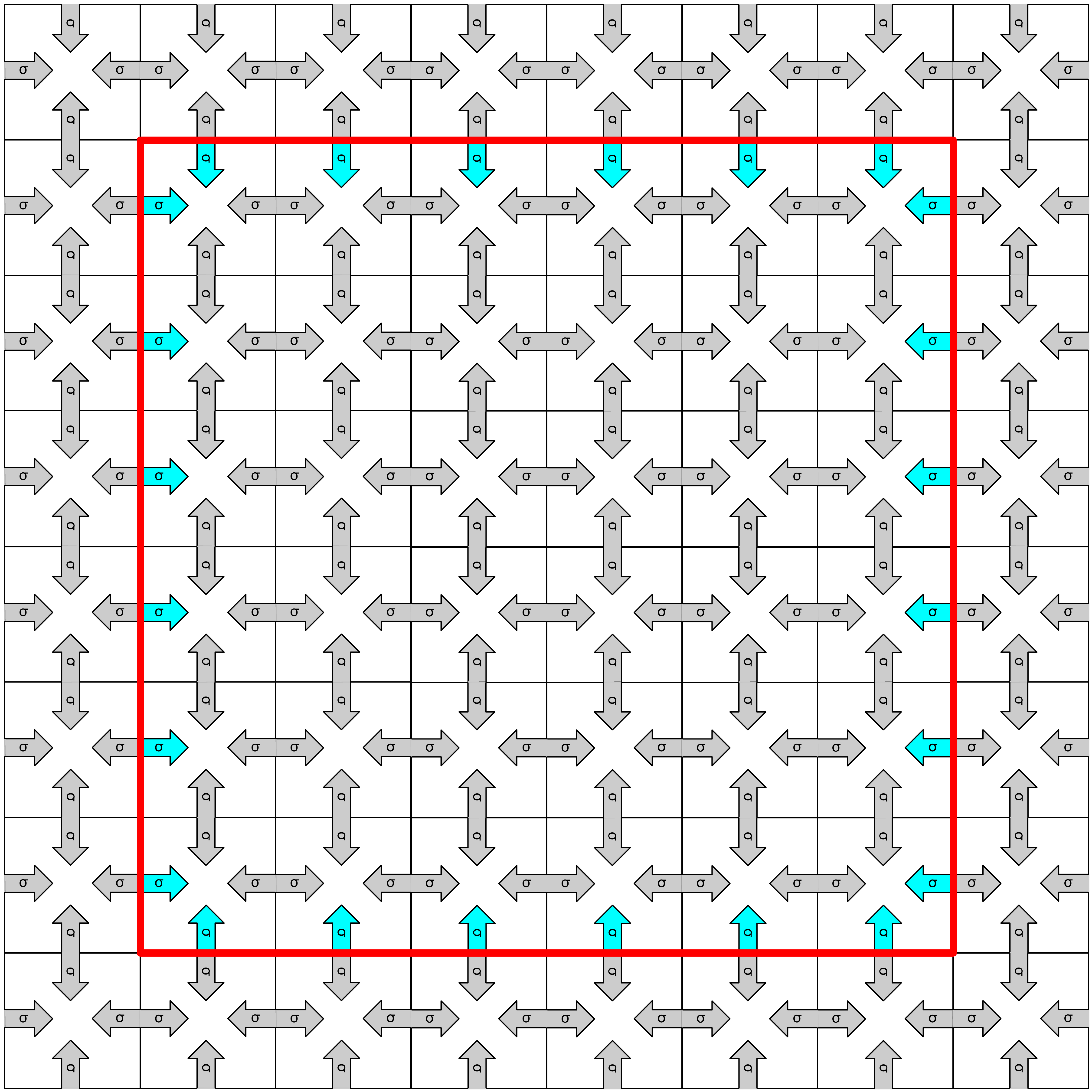}
    (b) Wasserstein Smoothing
\end{minipage} \caption{Schematic diagram showing the difference between Laplace and Wasserstein smoothing on the variance of the aggregate pixel intensity in a square region, outlined in red. See the text of Section \ref{sec:cmplaplace}. In both figures, pixels are represented as square tiles. In (a), noise on individual pixels is represented with circles, which are gray if they do \textit{not} contribute to the overall pixel intensity in the outlined region, but are cyan if they \textit{do} contribute. We see that the noise is proportional (in variance) to the area of the region. In (b), under Wasserstein smoothing, noise is represented by arrows between pixels which exchange intensity. Again, these are gray if they do not contribute to the overall pixel intensity in the outlined region, and cyan if they do contribute. Note that arrows in the interior do not contribute to the aggregate intensity, because equal values are added and subtracted from adjacent pixels. The noise is proportional (in variance) to the perimeter of the region. This provides a plausible intuition as to why base classifiers, when given noisy images, classify with higher accuracy on Wasserstein smoothed images compared to Laplace smoothed images, as seen empirically in Table \ref{MNISTdata}.} \label{fig:lapvwassexplain}
\end{figure*}
 \section{Experimental Results}
 In all experiments, we use 10,000 random noised samples to predict the smoothed classification of each image; to generate certificates, we first use 1000 samples to infer which class has highest smoothed score, and then 10,000 samples to lower-bound this score. All probabilistic certificates and classifications are reported to $95\%$ confidence.  The  model architectures used for the base classifiers for each data set are the same as used in \cite{pmlr-v97-wong19a}. When reporting results, \textit{median certified accuracy} refers to the maximum radius $\rho_2$ such that at least $50\%$ of classifications for images in the data set are certified to be robust to at least this radius, and these certificates are for the correct ground truth class. If over $50\%$ of images are not certified for the correct class, this statistic is reported as  $N/A$.
 \subsection{Comparison to naive Laplace Smoothing} \label{sec:cmplaplace}
 Note that one can derive a trivial but sometimes tight bound, that, under any $L_p$ distance metric, if $W_1(\vx,\tbx) \leq \rho/2$, then $\|\vx-\tbx\|_1 \leq \rho$.  (See Corollary \ref{laplace-corol} in the appendix.)  This enables us to write a condition for $\rho_2$-radius Wasserstein certified robustness by applying Laplace smoothing directly, and simply converting the certificate. In our notation, this condition is:
 \begin{equation}
    \barbf_{i}^{\text{Laplace}}(\vx) \geq e^{4\sqrt{2}\rho_2 / \sigma} (1- \barbf^{\text{Laplace}}_{i}(\vx) )
\end{equation}
where $\barbf^{\text{Laplace}}(\vx) $ is a smoothed classifier with Laplace noise added to every pixel independently. It may appear as if our Wasserstein-smoothed bound should only be an improvement over this bound by a factor of $\sqrt{2}$ in the certified radius $\rho_2$. However, as shown in Table \ref{MNISTdata}, we in fact improve our certificates by a larger factor. This is because, for a fixed noise standard deviation, the base classifier is able to achieve a higher accuracy after adding noise in the flow-domain, compared to adding noise directly to the pixels. When adding noise in the flow-domain, we add and subtract noise in equal amounts between adjacent pixels, preserving more information for the base classifier. 

To give a concrete example, consider some $k\times k$ square patch of an image. Suppose that the overall aggregate pixel intensity in this patch (i.e. the sum of the pixel values) is a salient feature for classification (This is a highly plausible situation: for example, in MNIST, this may indicate whether or not some region of an image is occupied by part of a digit.) Let us call this feature $\mu$, and calculate the variance of $\mu$ in smoothing samples under Laplace and Wasserstein smoothing, both with variance $\sigma^2$. Under Laplace smoothing (Figure \ref{fig:lapvwassexplain}-a), $k^2$ independent instances of Laplace noise are added to $\mu$, so the resulting variance will be $k^2\sigma^2$: this is proportional to the area of the region. In the case of Wasserstein smoothing, by contrast, probability mass exchanged between between pixels in the interior of the patch has no effect on the aggregate quantity $\mu$. Instead, only noise on the perimeter will affect the total feature value $\mu$: the variance is therefore $4k\sigma^2$ (Figure \ref{fig:lapvwassexplain}-b). Wasserstein smoothing then reduces the effective noise variance on the feature $\mu$ by a factor of $O(k)$. 

\subsection{Empirical adversarial accuracy} \label{sec:emp-acc}
\begin{figure}[t]
    \centering
    \includegraphics[width=0.48\textwidth,trim={0.8cm 0 1.7cm 1cm},clip]{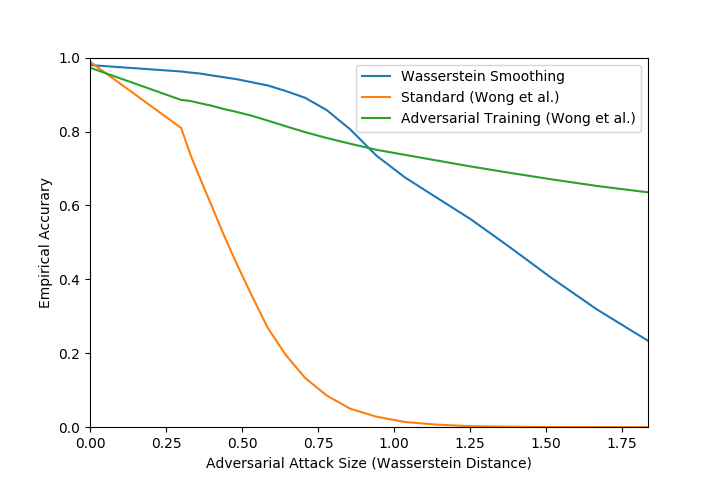}
    \caption{Comparison of empirical robustness on MNIST to models from \citep{pmlr-v97-wong19a}.  Wasserstein smoothing with $\sigma = 0.01$ (This is the amount of  noise which maximizes certified robustness, as seen in Table \ref{MNISTdata}.)  }
    \label{fig:mnist_emp}
\end{figure}
We measure the performance of our smoothed classifier against the Wasserstein-metric adversarial attack proposed in \cite{pmlr-v97-wong19a}, and compare to models tested in that work. Results are presented in Figure \ref{fig:mnist_emp}. For testing, we use the same attack parameters as in \cite{pmlr-v97-wong19a}: the 'Standard" and 'Adversarial Training' results are therefore replications of the experiments from that paper, using the publicly available code and pretrained models.

In order to attack our hard smoothed classifier, we adapt the method proposed by \cite{salman2019provably}: in particular, note that we cannot directly calculate the gradient of the classification loss with respect to the image for a \textit{hard} smoothed classifier, because the derivatives of the logits of the base classifier are not propagated. Therefore, we must instead attack a \textit{soft} smooth classifier: we take the expectation over samples of the \textit{softmaxed} logits of the base classifier, instead of the final classification output. In each step of the attack, we use 128 noised samples to estimate this gradient, as used in \cite{salman2019provably}.

In the attack proposed by \cite{pmlr-v97-wong19a}, the images are attacked over 200 iterations of projected gradient descent, projected onto a Wasserstein ball, with the radius of the ball every 10 iterations. The attack succeeds, and the final radius is recorded, once the classifier misclassifies the image. In order to preserve as much of the structure (and code) of the attack as possible to provide a fair comparison, it is thus necessary for us to evaluate each image using our hard classifier, with the full 10,000 smoothing samples, at each iteration of the attack. We  count the classifier abstaining as a misclassification for these experiments. However, note that this may somewhat underestimate the true robustness of our classifier: recall that our classifier is nondeterministic; therefore, because we are repeatedly evaluating the classifier and reporting a perturbed image as adversarial the first time it is missclassified, we may tend to over-count misclassifications. However, because we are using a large number of noise samples to generate our classifications, this is only likely to happen with examples which are close to being adversarial. Still, the presented data should be regarded as a lower bound on the true accuracy under attack of our Wasserstein smoothed classifier.
\begin{table*} [h!]
\centering
\begin{tabular}{|c|c|c|c|}
\hline
Noise standard deviation& Classification accuracy&Median certified& Base Classifier \\
$\sigma$&(Percent abstained)&robustness& Accuracy\\
\hline
0.00005&87.01(00.24)&0.000101&86.02\\
0.0001&83.39(00.42)&0.000179&82.08\\
0.0002&77.57(00.66)&0.000223&75.46\\
0.0005&68.75(01.01)&0.000209&65.12\\
0.001&61.65(01.77)&0.000127&57.03\\
\hline
\end{tabular}
\caption{Certified Wasserstein Accuracy of Wasserstein smoothing on CIFAR10} \label{CIFARdata}
\end{table*}

In Figure \ref{fig:mnist_emp}, we note two things: first, our Wasserstein smoothing technique appears to be an effective empirical defence against Wasserstein adversarial attacks, compared to an unprotected ('Standard')  network. (It is also more robust than the binarized and $L_{\infty}$-robust models tested by \cite{pmlr-v97-wong19a}: see appendix.) However, for large perturbations, our defence is less effective than the adversarial training defence proposed by \cite{pmlr-v97-wong19a}. This suggests a promising direction for future work: \cite{salman2019provably} proposed an adversarial training method for smoothed classifiers, which could be applied in this case. Note however that both Wasserstein adversarial attacks and smoothed adversarial training are computationally expensive, so this may require  significant computational resources. \\
Second, the median radius of attack to which our smoothed classifier is empirically robust is larger than the median certified robustness of our smoothed classifier by two orders of magnitude. This calls for future work both to develop improved robustness certificates as well as to develop more effective attacks in the Wasserstein metric.

\subsection{Experiments on color images (CIFAR-10)} \label{sec:color}
\begin{figure}[t]
    \centering
    \includegraphics[width=0.48\textwidth,trim={0.8cm 0 1.7cm 1cm},clip]{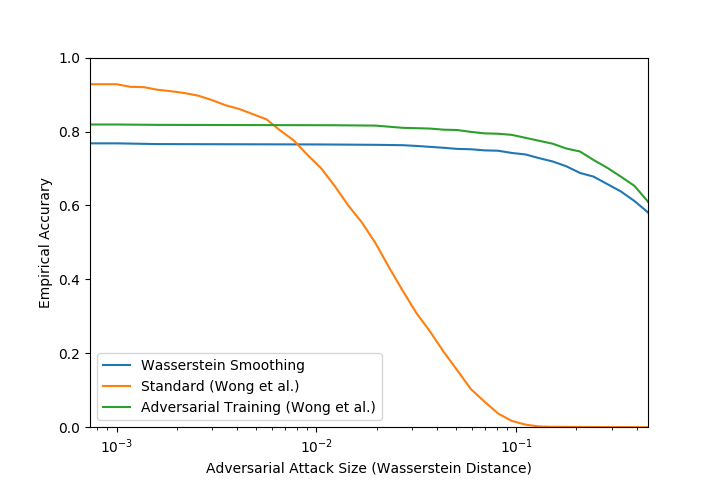}
    \caption{Comparison of empirical robustness on CIFAR-10 to models from \citep{pmlr-v97-wong19a}. Wasserstein smoothing is with $\sigma = 0.0002$. (This is the amount of noise which maximizes certified robustness, as seen in Table \ref{CIFARdata}.) Note that we test on a random sample of 1000 images from CIFAR-10, rather than the entire data set. }
    \label{fig:cifar_emp}
\end{figure}
\cite{pmlr-v97-wong19a} also apply their attack to color images in CIFAR-10. In this case, the attack does not transport probability mass between color channels: therefore, in our defence, it is sufficient to add noise in the flow domain to each channel independently to certify robustness (See Corollary \ref{threeD} in the appendix for a proof of the validity of this method). Certificates are presented in Table \ref{CIFARdata}, while empirical robustness is as Figure \ref{fig:cifar_emp}. Again, we compare directly to models from \cite{pmlr-v97-wong19a}. We note that again, empirically, our model significantly outperforms an unprotected model, but is not as robust as a model trained adversarially. We also note that the certified robustness is orders of magnitude smaller than computed for MNIST: however, the unprotected model is also significantly less robust empirically than the equivalent MNIST model.
\section{Conclusion}
In this paper, we developed a smoothing-based certifiably robust defence for Wasserstein-metric adversarial examples. To do this, we add noise in the space of possible flows of pixel intensity between images.  To our knowledge, this is the first certified defence method specifically tailored to the Wasserstein threat model. Our method proves to be an effective practical defence against Wasserstein adversarial attacks, with significantly improved empirical adversarial robustness compared to a baseline model. 
\bibliographystyle{humannat} 
\nocite{*}
\bibliography{ref}
\clearpage
\appendix 
\section{Proofs}
\begin{lemmarestate}
For any normalized probability distributions $\vx, \vx' \in[0,1]^{n \times m}$, there exists at least one $\bm{\delta}$ such that $\vx' =  \Delta(\vx,\bm{\delta}) $. Furthermore:
\begin{equation} 
    \min_{ \bm{\delta}: \,\, \vx' =  \Delta(\vx,\bm{\delta}) } \|\bm{\delta}\|_1 = W_{1}(\vx,\vx') \label{flowminrestate}
\end{equation}
Where $W_1$ denotes the 1-Wasserstein metric, using $L_1$ distance as the underlying distance metric.
\end{lemmarestate}
\begin{proof}
We first show the equivalence of the above minimization problem with the linear program proposed by \cite{ling2007efficient}, restated here:
\begin{align} \label{wassl1lprepeat}
    W_1(\vx,\vx') =\min_{\bg} \sum _{\substack{(i,j)}}~ \sum _{\substack{(i',j')\in \cN\left((i,j)\right)}} \bg_{(i,j),(i',j')}
\end{align}    
where $\bg \geq 0$ and $\forall (i,j)$, 
\begin{align}  
    \sum _{\substack{(i',j')\in \cN\left((i,j)\right)}}  \bg_{(i,j),(i',j')} - \bg_{(i',j'),(i,j)} = \bx'_{i,j} -\bx_{i,j}\nonumber
\end{align}
It suffices to show that (1) there is a transformation from the variables $\bg$ in Equation \ref{wassl1lprepeat} to the variables $\bm{\delta}$ in Equation \ref{flowminrestate}, such that  all points which are feasible in Equation \ref{wassl1lprepeat} are feasible in \ref{flowminrestate} and the minimization objective in Equation \ref{flowminrestate} is less than or equal to the minimization objective in Equation \ref{wassl1lprepeat}, and (2) there is a transformation from the variables $\bm{\delta}$ in Equation \ref{flowminrestate} to the variables $\bg$ in Equation \ref{wassl1lprepeat}, such that all points which are feasible in Equation \ref{flowminrestate} are feasible in Equation \ref{wassl1lprepeat} and the minimization objective in Equation \ref{wassl1lprepeat} is less than or equal to the minimization objective in Equation \ref{flowminrestate}.\\
We start with (1). We give the transformation as:
\begin{equation} \label{wasstransform1}
\begin{split}
    \bm{\delta}^{\text{vert.}}_{i, j} := \bg_{(i,j),(i+1,j)} -  \bg_{(i+1,j),(i,j)}\\
    \bm{\delta}^{\text{horiz.}}_{i, j} := \bg_{(i,j),(i,j+1)} -  \bg_{(i,j+1),(i,j)} 
\end{split}
\end{equation}
Where we let $\bg_{(n,j),(n+1,j)} = \bg_{(n+1,j),(n,j)} =  \bg_{(i,m+1),(i,m)}=  \bg_{(i,m),(i,m+1)} = 0 $. To show feasibility, we write out fully the flow constraint of Equation \ref{wassl1lprepeat}:\\
\begin{equation}
\begin{split}
      \bg_{(i,j),(i+1,j)} -  \bg_{(i+1,j),(i,j)} +&\\
      \bg_{(i,j),(i-1,j)} -  \bg_{(i-1,j),(i,j)} +&\\
      \bg_{(i,j),(i,j+1)} -  \bg_{(i,j+1),(i,j)}  +& \\
      \bg_{(i,j),(i,j-1)} -  \bg_{(i,j-1),(i,j)}   =& \bx'_{i,j} -\bx_{i,j}
\end{split}
\end{equation}
Substituting in Equation \ref{wasstransform1}:
\begin{equation}
\begin{split}
     \bm{\delta}^{\text{vert.}}_{i, j} +
     -\bm{\delta}^{\text{vert.}}_{i-1, j} +
      \bm{\delta}^{\text{horiz.}}_{i, j}  +
     -\bm{\delta}^{\text{horiz.}}_{i, j-1}   = \bx'_{i,j} -\bx_{i,j}
\end{split}
\end{equation}
But by Definition \ref{flowdef}, this is exactly:
\begin{equation}
   \Delta(\vx,\bm{\delta})_{i,j} = \vx'_{i,j}  
\end{equation}
Which is the sole constraint in Equation \ref{flowminrestate}: then any solution which is feasible in Equation \ref{wassl1lprepeat} is feasible in Equation \ref{flowminrestate}. Also note that:
\begin{equation}
\begin{split}
 \|\bm{\delta}\|_1= &\sum_{i,j} |\bm{\delta}_{i,j}^{\text{vert.}}| +  |\bm{\delta}_{i,j}^{\text{horiz.}}|\\
  \leq& \sum_{i,j}  |\bg_{(i,j),(i+1,j)}| + |\bg_{(i+1,j),(i,j)}|\\
   &+   |\bg_{(i,j),(i,j+1)}| + |\bg_{(i,j+1),(i,j)}|\\
   =& \sum_{i,j}  \bg_{(i,j),(i+1,j)} + \bg_{(i+1,j),(i,j)} \\
    &+   \bg_{(i,j),(i,j+1)} + \bg_{(i,j+1),(i,j)}\\
   =& \sum_{i,j}  \bg_{(i,j),(i+1,j)} +   \bg_{(i,j),(i,j+1)} \\
   &+\sum_{i,j} \bg_{(i+1,j),(i,j)} + \bg_{(i,j+1),(i,j)} \\
  =& \sum_{i,j}  \bg_{(i,j),(i+1,j)} +   \bg_{(i,j),(i,j+1)} \\
   &+\sum_{i,j} \bg_{(i,j),(i-1,j)} + \bg_{(i,j),(i,j-1)} \\
  = &\sum _{\substack{(i,j)}}~ \sum _{\substack{(i',j')\in \cN\left((i,j)\right)}} \bg_{(i,j),(i',j')}
\end{split}
\end{equation}
Where the inequality follows from triangle inequality applied to Equation \ref{wasstransform1}, and in the second sum in the fourth line, we exploit the fact that $\bg_{(n,j),(n+1,j)} = \bg_{(n+1,j),(n,j)} =  \bg_{(i,m+1),(i,m)}=  \bg_{(i,m),(i,m+1)} = 0 $ to shift indices. This shows that the minimization objective in Equation \ref{flowminrestate} is less than or equal to the minimization objective in Equation \ref{wassl1lprepeat}.\\
Moving on to (2), we give the transformation as:
\begin{equation} \label{wasstransform2}
\begin{split}
   \bg_{(i,j),(i+1,j)}& := \max( \bm{\delta}^{\text{vert.}}_{i, j}, 0)\\
   \bg_{(i,j),(i-1,j)}& := \max( -\bm{\delta}^{\text{vert.}}_{i-1, j}, 0)\\
   \bg_{(i,j),(i,j+1)}& := \max( \bm{\delta}^{\text{horiz.}}_{i, j}, 0)\\
   \bg_{(i,j),(i,j-1)}& := \max( -\bm{\delta}^{\text{horiz.}}_{i, j-1}, 0)\\
\end{split}
\end{equation}
Note that the non-negativity constraint of Equation \ref{wassl1lprepeat} is automatically satisfied by the form of these definitions. Shifting indices, we also have:
\begin{equation}
\begin{split}
   \bg_{(i-1,j),(i,j)}& = \max( \bm{\delta}^{\text{vert.}}_{i-1, j}, 0)\\
   \bg_{(i+1,j),(i,j)}& = \max( -\bm{\delta}^{\text{vert.}}_{i, j}, 0)\\
   \bg_{(i,j-1),(i,j)}& = \max( \bm{\delta}^{\text{horiz.}}_{i, j-1}, 0)\\
   \bg_{(i,j+1),(i,j)}& = \max( -\bm{\delta}^{\text{horiz.}}_{i, j}, 0)\\
\end{split}
\end{equation}
From the constraint on Equation \ref{flowminrestate}, we have:
\begin{equation}
\begin{split}
     \bx'_{i,j} -\bx_{i,j} = & \bm{\delta}^{\text{vert.}}_{i, j} +\\
    & -\bm{\delta}^{\text{vert.}}_{i-1, j} +\\
    &  \bm{\delta}^{\text{horiz.}}_{i, j}  +\\
     &-\bm{\delta}^{\text{horiz.}}_{i, j-1}   \\
     =& \max( \bm{\delta}^{\text{vert.}}_{i, j}, 0) -  \max( -\bm{\delta}^{\text{vert.}}_{i, j}, 0) +\\
      & \max( -\bm{\delta}^{\text{vert.}}_{i-1, j} -\max( \bm{\delta}^{\text{vert.}}_{i-1, j}, 0)  , 0)+\\
& \max( \bm{\delta}^{\text{horiz.}}_{i, j}, 0) -  \max( -\bm{\delta}^{\text{horiz.}}_{i, j}, 0) +\\
      & \max( -\bm{\delta}^{\text{horiz.}}_{i, j-1}, 0) -\max( \bm{\delta}^{\text{horiz.}}_{i, j-1}, 0) \\
      =& \bg_{(i,j),(i+1,j)} - \bg_{(i+1,j),(i,j)}+\\
       & \bg_{(i,j),(i-1,j)}  -\bg_{(i-1,j),(i,j)}+\\
        &\bg_{(i,j),(i,j+1)} - \bg_{(i,j+1),(i,j)}+\\
        &\bg_{(i,j),(i,j-1)}  -\bg_{(i,j-1),(i,j)}\\
\end{split}
\end{equation}
Which is exactly the second constraint of Equation \ref{wassl1lprepeat}: then any solution which is feasible in Equation \ref{wassl1lprepeat} is feasible in Equation \ref{flowminrestate}. Also note that:
\begin{equation}
    \begin{split}
&\sum _{\substack{(i,j)}}~ \sum _{\substack{(i',j')\in \cN\left((i,j)\right)}} \bg_{(i,j),(i',j')} \\
  =& \sum_{i,j}  \bg_{(i,j),(i+1,j)} +   \bg_{(i,j),(i,j+1)} \\
   &+\sum_{i,j} \bg_{(i,j),(i-1,j)} + \bg_{(i,j),(i,j-1)} \\
  =& \sum_{i,j} \max( \bm{\delta}^{\text{vert.}}_{i, j}, 0)+   \max( \bm{\delta}^{\text{horiz.}}_{i, j}, 0)\\
   &+\sum_{i,j}\max( -\bm{\delta}^{\text{vert.}}_{i-1, j}, 0)+ \max( -\bm{\delta}^{\text{horiz.}}_{i, j-1}, 0) \\ 
   =& \sum_{i,j} \max( \bm{\delta}^{\text{vert.}}_{i, j}, 0)+\max( -\bm{\delta}^{\text{vert.}}_{i, j}, 0)\\
   &+\sum_{i,j}  \max( \bm{\delta}^{\text{horiz.}}_{i, j}, 0)+  \max( -\bm{\delta}^{\text{horiz.}}_{i, j}, 0)\\
   =& \sum_{i,j} |\bm{\delta}^{\text{vert.}}_{i, j}| +  |\bm{\delta}^{\text{horiz.}}_{i, j}|\\
   =&  \|\bm{\delta}\|_1\\
    \end{split}
\end{equation}
Where we again exploit the fact that $\bg_{(n,j),(n+1,j)} = \bg_{(n+1,j),(n,j)} =  \bg_{(i,m+1),(i,m)}=  \bg_{(i,m),(i,m+1)} = 0 $ to shift indices, in the fourth line. This shows that the minimization objective in Equation \ref{wassl1lprepeat} is less than or equal to the minimization objective in Equation \ref{flowminrestate}, completing (2).\\
Finally, now that we have shown that Equations \ref{flowminrestate} and \ref{wassl1lprepeat} are in fact equivalent minimizations (i.e., we have proven Equation \ref{flowminrestate} correct), we would like to show that there is always a feasible solution to  \ref{flowminrestate}, as claimed. By the above transformations, it suffices to show that there is always a feasible solution to Equation \ref{wassl1lprepeat}. \cite{ling2007efficient} show that any feasible solution the the general Wasserstein minimization LP (Definition \ref{wassdef}) can be transformed into a solution to Equation \ref{wassl1lprepeat}, so it suffices to show that the LP in Definition \ref{wassdef} always has a feasible solution. This is trivially satisfied by taking $\Pi=\vx(\vx')^T$, where we note that $\vx$, a probability distribution, is non-negative. 
\end{proof}
\begin{theoremrestate} 
Consider a normalized probability distribution $\vx \in [0,1]^{n \times m}$, and a classification score function $\vf:  \R^{n \times m} \rightarrow [0,1]^k $. Let $\barbf$ refer to the Wasserstein-smoothed classification function:
\begin{equation}
   \barbf(\vx) = \underset{\bm{\delta} \sim \mathcal{L}(\sigma) }{\E}\left[ \vf (\Delta(\vx, \bm{\delta}))\right]
\end{equation}
Let $i$ be the class assignment of $\vx$ using the smoothed classifier $\barbf$ (i.e. $i = \arg\max_{i'} \barbf_{i'}(\vx)$). If
\begin{equation}
    \barbf_{i}(\vx) \geq e^{2\sqrt{2}\rho / \sigma} \max_{i' \neq i} \barbf_{i'}(\vx)
\end{equation}
Then for any perturbed probability distribution $\tbx$ such that $W_1(\vx,\tbx) \leq \rho$:
\begin{equation}
    \barbf_{i}(\tbx) \geq \max_{i' \neq i} \barbf_{i'}(\tbx)
\end{equation}
\end{theoremrestate}
\begin{proof}
Let $\vu$ be the uniform probability vector. As a consequence of Lemma \ref{wassflow}, for any distribution  $\vx$, there exists a nonempty set of local flow plans $S_\vx$:
\begin{equation}
    S_\vx = \{\bm{\delta}| \vx = \Delta(\vu, \bm{\delta})\}
\end{equation}
Also, we may define a version of the classifier $\vf$ on the local flow plan domain:
\begin{equation}
    \vf^{\text{flow}}(\bm{\delta}) = \vf(\Delta(\vu, \bm{\delta}))
\end{equation}
Let $\bm{\delta}_x$ be an arbitrary element in $S_x$, and consider any perturbed $\tbx$ such that $W_1(\vx,\tbx) \leq \rho$. By Theorem \ref{wassflow}:
\begin{equation}
    \min_{ \bm{\delta}: \,\, \tbx =  \Delta(\vx,\bm{\delta}) } \|\bm{\delta}\|_1 = W_{1}(\vx,\tbx)
\end{equation}
Then, using Equation \ref{flowadd}:
\begin{equation}
     \min_{ \bm{\delta}: \,\, \tbx =  \Delta(\vu, \bm{\delta}_\vx + \bm{\delta}) } \|\bm{\delta}\|_1 = W_{1}(\vx,\tbx)
\end{equation}
Let the minimum be achieved at $\bm{\delta}^*$.
Making a change of variables ($\bm{\delta}_{\tbx} = \bm{\delta}^* + \bm{\delta}_x$), we have:
\begin{equation}
      \|\bm{\delta}_{\tbx}-\bm{\delta}_\vx\|_1 = W_{1}(\vx,\tbx)\quad\text{where  } \tbx =  \Delta(\vu, \bm{\delta}_{\tbx})
\end{equation}
Note that for any $\vx'$ (for $\bm{\delta}' \sim \mathcal{L}(\sigma) $) :
\begin{equation} \label{flowequivance}
\begin{split}
   \barbf(\vx') =&{\E}\left[ \vf (\Delta(\vx', \bm{\delta}')\right]\\
   =& {\E}\left[ \vf (\Delta(\vu, \bm{\delta}_{\vx'}+ \bm{\delta}'))\right] \\
   =& {\E}\left[ \vf^{\text{flow}} (\bm{\delta}_{\vx'}+ \bm{\delta}'))\right]  
\end{split}
\end{equation}
We can now apply Proposition 1 from \cite{lecuyer2018certified}, restated here:
\begin{proposition*}
Consider a vector $\vv \in \R^{d}$, and a classification score function $\vh:  \R^{d} \rightarrow [0,1]^k $. Let $\epsilon \sim \laplace(0, \sigma)^{d}$, and let $i$ be the class assignment of $\vv$ using a Laplace-smoothed version of the classifier $\vh$:
\begin{equation}
i = \arg\max_{i'} \underset{\epsilon}{\E}\left[ \vh_{i'}(\vv+\epsilon) \right]
\end{equation}
If:
\begin{equation}
   \underset{\epsilon}{\E}\left[ \vh_{i}(\vv+\epsilon) \right] \geq e^{2\sqrt{2}\rho / \sigma} \max_{i' \neq i} \underset{\epsilon }{\E}\left[ \vh_{i'}(\vv + \epsilon) \right]
\end{equation}
Then for any perturbed probability distribution $\tbv$ such that $\|\vv-\tbv\|_1 \leq \rho$:
\begin{equation}
\underset{\epsilon}{\E}\left[ \vh_{i}(\tbv+\epsilon) \right] \geq \max_{i' \neq i} \underset{\epsilon }{\E}\left[ \vh_{i'}(\tbv + \epsilon) \right]
\end{equation}
\end{proposition*}
We apply this proposition to $\vf^{\text{flow}}$, noting that $ \|\bm{\delta}_{\tbx}-\bm{\delta}_\vx\|_1 = W_{1}(\vx,\tbx) \leq \rho$:
\begin{equation}
\begin{split}
    \underset{\bm{\delta}'}{\E}\left[ \vf^{\text{flow}}_i (\bm{\delta}_{\vx}+ \bm{\delta}'))\right] \geq e^{2\sqrt{2}\rho / \sigma} \max_{i' \neq i} \underset{\bm{\delta}' }{\E}\left[ \vf^{\text{flow}}_{i'} (\bm{\delta}_{\vx}+ \bm{\delta}'))\right]\\
 \implies \underset{\bm{\delta}'}{\E}\left[ \vf^{\text{flow}}_i (\bm{\delta}_{\tbx}+ \bm{\delta}'))\right] \geq  \max_{i' \neq i} \underset{\bm{\delta}' }{\E}\left[ \vf^{\text{flow}}_{i'} (\bm{\delta}_{\tbx}+ \bm{\delta}'))\right] 
\end{split}
\end{equation}
Then, using Equation \ref{flowequivance}:
\begin{equation}
\begin{split}
    \barbf_i(\vx) \geq e^{2\sqrt{2}\rho / \sigma} \max_{i' \neq i} \barbf_{i'}(\vx) &\implies \\
\barbf_i(\tbx) \geq  \max_{i' \neq i} \barbf_{i'}(\tbx) &
\end{split}
\end{equation}
Which was to be proven.
\end{proof}
\begin{corollary} \label{laplace-corol}
For any normalized probability distributions $\vx, \vx' \in[0,1]^{n \times m}$, if $W_1(\vx, \vx') \leq \rho / 2$, then $\|\vx - \vx'\|_1 \leq \rho$, where $W_1$ is the 1-Wasserstein metric using any $L_p$ norm as the underlying distance metric. Furthermore, there exist distributions where these inequalities are tight.
\end{corollary}
\begin{proof}
Let $\Pi$ indicate the optimal transport plan between $\vx$ and $\vx'$. From Definition \ref{wassdef}, we have  $\Pi \1 = \vx$ and $\Pi^T\1 = \vx'$. Then:
\begin{equation}
    (\Pi^T - \Pi )\1 = \vx'- \vx
\end{equation}
Let $\Pi'$ represent a modified version of $\Pi$, with the diagonal elements set to zero. Note that $<\Pi',C> = <\Pi,C>$ and  $\Pi^T - \Pi  = (\Pi')^T - \Pi' $. Then, using triangle inequality:
\begin{equation}
\begin{split}
   & \|(\Pi')^T\1 \|_1 +  \|(\Pi')\1 \|_1\\
   \geq &\| ((\Pi')^T - \Pi' )\1\|_1 \\
   =& \|\vx'- \vx\|_1
\end{split}
\end{equation}
Because the elements of $\Pi'$ are non-negative, this is simply:
\begin{equation}
 2\sum_{i,j} \Pi'_{i,j} \geq \| ((\Pi')^T - \Pi' )\1\|_1 = \|\vx'- \vx\|_1
\end{equation}
Then, because the (non-diagonal) elements of $C$ are at least $1$ for any $L_p$ norm, we have,
\begin{equation}
2<\Pi',C>\geq 2\sum_{i,j} \Pi'_{i,j} \geq  \|\vx'- \vx\|_1
\end{equation}
Because $<\Pi',C> = <\Pi,C> = W_1(\vx,\vx')$, this means that  $\|\vx'- \vx\|_1 \leq 2 W_1(\vx,\vx') \leq \rho$, which was to be proven.
Note that this inequality can be tight. For example, let $\vx$ be the distribution where the entire probability mass is at position $(i,j)$, and $\vx'$ be the distribution where the probability mass is equally split between at positions $(i,j)$ and $(i+1, j)$. (In other words, $\vx_{(i,j)} = 1, \vx'_{(i,j)} = .5, \vx'_{(i+1,j)} = .5$). In this case, $\|\vx'- \vx\|_1 =1$, $W_1(\vx,\vx') =.5$.
\end{proof}
\begin{corollary} \label{threeD}
Consider a color image with three channels, denoted $\bx = [\bx^R,\bx^G,\bx^B]$, normalized such that $\sum_{(i,j)}  \bx^R_{(i,j)}+\bx^G_{(i,j)}+\bx^B_{(i,j)} = 1$. Consider a perturbed image $\tbx$ such that  $\forall K \in \{R,G,B\}, \, \sum_{(i,j)}  \bx^K_{(i,j)} = \sum_{(i,j)} \tbx^K_{(i,j)}$. Let $W_1(\bx,\tbx)$ denote the 1-Wasserstein distance (with $L_1$ distance metric) between $\bx$ and $\tbx$, where, when determining the minimum transport plan, transport between channels is not permitted. Using this definition, let $W_1(\vx,\tbx) \leq \rho$. Define:
\begin{equation} \label{defrgb}
\begin{split}
        \bm{\delta} &= \{\bm{\delta}^R,\bm{\delta}^G, \bm{\delta}^B\}\\
        \Delta(\vx,\bm{\delta})&=  \{\Delta(\vx^R,\bm{\delta^R}),\Delta(\vx^G,\bm{\delta^G}),\Delta(\vx^B,\bm{\delta^B})\}
\end{split}
\end{equation}
and let $\mathcal{L}^{\text{color}}(\sigma)$ represent independent draws of Laplace noise each with standard deviation $\sigma$ in the shape of $\bm{\delta}$. Then if 
\begin{equation}
    \barbf_{i}(\vx) \geq e^{2\sqrt{2}\rho / \sigma} \max_{i' \neq i} \barbf_{i'}(\vx)
\end{equation}
then
\begin{equation}
    \barbf_{i}(\tbx) \geq \max_{i' \neq i} \barbf_{i'}(\tbx).
\end{equation}
\end{corollary}
\begin{proof}
Let the mass in each channel be denoted $s_{K}$:
\begin{equation}
    s_{K} :=  \sum_{(i,j)}  \bx^K_{(i,j)} = \sum_{(i,j)} \tbx^K_{(i,j)}
\end{equation}
Consider the formulation of Wasserstein distance given in Definition \ref{wassdef}. If we represent the elements of $\vx$ as a vector by concatenating the elements of $\bm{\delta}^R,\bm{\delta}^G,$ and  $\bm{\delta}^B$, then the restriction that there is no flow between channels amounts to the requirement that $\Pi$ is block-diagonal:
\begin{equation} \label{piblock}
\Pi = 
\begin{bmatrix}
    \Pi^R  & \0 & \0 \\
    \0  & \Pi^G & \0 \\
    \0 & \0 & \Pi^B \\
\end{bmatrix}
\end{equation}
Let $C^{1,1}$ represent the standard cost matrix for 1-Wasserstein transport  (with $L_1$ distance metric). Because the cost of transport within each channel is the same for standard 1-Wasserstein transport  (with $L_1$ distance metric), we have:
\begin{equation}
C = 
\begin{bmatrix}
    C^{1,1}  & \0 & \0 \\
    \0  & C^{1,1} & \0 \\
    \0 & \0 & C^{1,1} \\
\end{bmatrix}
\end{equation}
Then we have:
\begin{equation}
    <\Pi,C> = <\Pi^R,C^{1,1}> +<\Pi^G,C^{1,1}> +<\Pi^B,C^{1,1}> 
\end{equation}
And by Equation \ref{piblock}, the constraints also factorize out:
\begin{equation}
\begin{split}
       \Pi^R\1=\vx^R, \,\, (\Pi^R)^T\1=\tbx^R, \\\Pi^B\1=\vx^B, \,\, (\Pi^B)^T\1=\tbx^B,  \\\Pi^G\1=\vx^G, \,\, (\Pi^G)^T\1=\tbx^G  
\end{split}
\end{equation}
Then the variables of each $\Pi^K$ are separable (in that they appear together in the objective only in the sum and share no constraints). We can then factorize the minimization:
\begin{align}
    W_1(\vx,\tbx) = \sum_{K}\,\,& \min_{\Pi^K \in \R_{+}^{(n\cdot m) \times (n\cdot m)}} <\Pi^K,C^{(1,1)}>, \\
   \forall K, \,\, &\Pi^K\1=\vx^K, \,\, (\Pi^K)^T\1=\tbx^K \\
\end{align}
We can transform each $x^K$ into a normalized probability distribution by scaling it by a factor of $1/s_K$. We similarly scale each $\Pi^K$ :
\begin{equation}
    \vx^K_{\text{sc.}} := \frac{\vx^K}{s_K}\,\,\,\,  \Pi^K_{\text{sc.}} := \frac{\Pi^K}{s_K}
\end{equation}
Then we have:
\begin{align}
    W_1(\vx,\tbx) = \sum_{K} s_K \cdot& \min_{\Pi^K_{\text{sc.}}  \in \R_{+}^{(n\cdot m) \times (n\cdot m)}} <\Pi^K_{\text{sc.}} ,C^{(1,1)}>, \\
   \forall K, \,\, &\Pi^K_{\text{sc.}} \1=\vx^K_{\text{sc.}} , \,\, (\Pi^K_{\text{sc.}})^T \1=\tbx^K_{\text{sc.}}  \\
\end{align}
But note that this is simply:
\begin{align}
    W_1(\vx,\tbx) = \sum_{K} s_K \cdot W_1(\vx^K_{\text{sc.}},\tbx^K_{\text{sc.}}) 
\end{align}
By Lemma \ref{wassflow}, this is:
\begin{equation}
    W_1(\vx,\tbx) = \sum_{K} s_K \cdot \min_{ \bm{\delta^K_{sc.}}: \,\, \tbx^K_{sc.} =  \Delta(\vx^K_{sc.},\bm{\delta}^K_{sc.}) } \|\bm{\delta}^K_{sc.}\|_1 
\end{equation}
By the linearity to scaling of $\Delta$ and the $L_1$ norm, this is simply:
\begin{equation}
    W_1(\vx,\tbx) = \sum_{K}\min_{ \bm{\delta^K}: \,\, \tbx^K =  \Delta(\vx^K,\bm{\delta}^K) } \|\bm{\delta}^K\|_1 
\end{equation}
Which, by Equation \ref{defrgb}, is simply,
\begin{equation}
    W_1(\vx,\tbx) = \min_{ \bm{\delta}: \,\, \tbx =  \Delta(\vx,\bm{\delta}) } \|\bm{\delta}\|_1 
\end{equation}
Then all of the mechanics of the proof of Theorem \ref{smoothwas} apply, and (avoiding unnecessary repetition), we conclude the result.
\end{proof}
\section{Comparison to other Defences in \cite{pmlr-v97-wong19a}}
\begin{figure}[h!]
    \centering
    \includegraphics[width=0.48\textwidth,trim={0.8cm 0 1.7cm 1cm},clip]{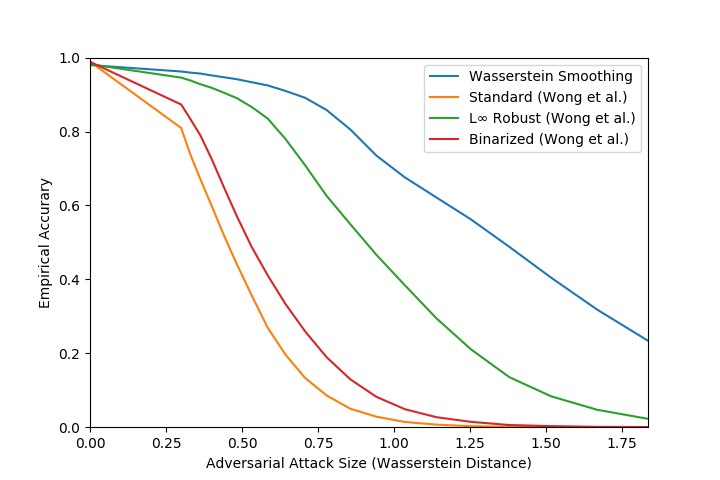}
    \caption{Comparison of empirical robustness on MNIST to additional defences from \citep{pmlr-v97-wong19a}, other than adversarial training. Randomized Smoothing shown here is Wasserstein smoothing with $\sigma = 0.01$. (This is the amount of noise which maximizes certified robustness, as seen in Table \ref{MNISTdata}.)  }
    \label{fig:mnist_emp_apdx}
\end{figure}
In addition to proposing adversarial training as a defence against Wasserstein Adversarial attacks, \cite{pmlr-v97-wong19a} also tests other defences. On MNIST, binarization of the input and using a provably $L_\infty$-robust classifier were also tested as defences: our randomized smoothing method is more effective than these methods at all attack magnitudes (see Figure \ref{fig:mnist_emp_apdx}). On CIFAR-10, \cite{pmlr-v97-wong19a} only tested a provably $L_\infty$-robust classifier as an additional defence: unfortunately, code was not provided for this model, so we did not attempt to replicate the results.
\section{Training Parameters}
In this paper, network architectures models used were identical to those used in \cite{pmlr-v97-wong19a}. Unless stated otherwise, all parameters of attacks are the same as used in that paper for each data set. For training smoothed models, we train the base classifier using standard cross-entropy loss on individual noised sample images, using the same noise distribution as used when performing smoothed classification. However, during training, rather than using the same image repeatedly while adding different noise (as at test time), we instead train with each image only once per epoch, with one noise draw. In fact, for computational efficiency and as suggested by \cite{lecuyer2018certified}, we re-use the same noise for each image in a batch. Training parameters are as follows (Tables \ref{tab:MNISTtrain}, \ref{tab:CIFARtrain}):
\begin{table}[h!]
    \centering
    \begin{tabular}{|c|c|}
        \hline
        Training Epochs & 200\\
        \hline
        Batch Size & 128\\
        \hline
        Optimizer & Stochastic Gradient \\
        &Descent with Momentum\\
        \hline
        Learning Rate & .001 \\
        \hline
        Momentum & 0.9 \\
        \hline
        $L_2$ Weight Penalty & 0.0005 \\
        \hline
    \end{tabular}
    \caption{Training Parameters for MNIST Experiments}
    \label{tab:MNISTtrain}
\end{table}
\begin{table}[h!]
    \centering
    \begin{tabular}{|c|c|}
        \hline
        Training Epochs & 200\\
        \hline
        Batch Size & 128\\
        \hline
        Training Set & Normalization, \\
        Preprocessing & Random Cropping (Padding:4)\\
        & and Random Horizontal Flip\\
        \hline
        Optimizer & Stochastic Gradient \\
        &Descent with Momentum\\
        \hline
        Learning Rate & .01 (Epochs 1-200) \\
        & .001 (Epochs 201-400)\\
        \hline
        Momentum & 0.9 \\
        \hline
        $L_2$ Weight Penalty & 0.0005 \\
        \hline
    \end{tabular}
    \caption{Training Parameters for CIFAR-10 Experiments}
    \label{tab:CIFARtrain}
\end{table}
\\\\.
\end{document}

%% file: math_commands.tex
\usepackage{amsmath,amsfonts,bm}

\def\eqref#1{equation~\ref{#1}}

\def\1{\bm{1}}
\def\0{\bm{0}}

\def\vf{{\bm{f}}}

\def\vh{{\bm{h}}}

\def\vr{{\bm{r}}}

\def\vu{{\bm{u}}}
\def\vv{{\bm{v}}}

\def\vx{{\bm{x}}}

\DeclareMathAlphabet{\mathsfit}{\encodingdefault}{\sfdefault}{m}{sl}
\SetMathAlphabet{\mathsfit}{bold}{\encodingdefault}{\sfdefault}{bx}{n}

\newcommand{\laplace}{\mathrm{Laplace}} 

\newcommand{\E}{\mathbb{E}}

\newcommand{\R}{\mathbb{R}}

\newcommand{\RR}{\mathbb{R}}

\newcommand{\bx}{\mathbf{x}}

\newcommand{\cN}{\mathcal{N}}

\newcommand{\tbf}{\tilde{\vf}}

\newcommand{\bg}{\mathbf{g}}

\newcommand{\barbf}{\bar{\vf}}

\newcommand{\tbx}{\tilde{\bm{x}}}
\newcommand{\tbv}{\tilde{\bm{v}}}